\documentclass[11pt,a4paper]{article}

\usepackage{inputenc,hyperref}
\usepackage[T1]{fontenc}
\usepackage{array}
\usepackage{booktabs}

\newcommand{\RR}{\mathbb{R}}

\providecommand{\keywords}[1]{\textbf{Keywords:} #1}

\usepackage{amsfonts,amsmath,amssymb,amsthm}
\usepackage{graphicx}
\usepackage{xcolor}
\usepackage{pifont}
\usepackage{times}
\usepackage{geometry}

\usepackage{hyperref}
\usepackage{subfigure}
\usepackage{pgf}
\usepackage{tikz}
\usepackage{url}
\usepackage{mathtools}
\usetikzlibrary{arrows,automata}
\usetikzlibrary{positioning}

\usepackage{algorithm}
\usepackage{algorithmic}

\newtheorem{Lemma}{Lemma}

\newcommand{\tabincell}[2]{\begin{tabular}{@{}#1@{}}#2\end{tabular}}  

\usepackage[square,numbers]{natbib}
\begin{document}

\title{Discriminative Principal Component Analysis:  \textsc{A Reverse Thinking}\footnote{\emph{Discriminative PCA}}
}

%% Group authors per affiliation:

%% or include affiliations in footnotes:
\author{Hanli Qiao\\
School of Science, Guilin University of Technology, China\\
\texttt{ hanlinqiao77@gmail.com}}
\date{}

\maketitle

\begin{abstract}
In this paper, we propose a novel approach named by \emph{Discriminative Principal Component Analysis} which is abbreviated as \textit{Discriminative PCA} in order to enhance separability of PCA by Linear Discriminant Analysis (LDA). The proposed method performs feature extraction by determining a linear projection that captures the most scattered discriminative information. The most innovation of \textit{Discriminative PCA} is performing PCA on discriminative matrix rather than original sample matrix. For calculating the required discriminative matrix under low complexity, we exploit LDA on a converted matrix to obtain within-class matrix and between-class matrix thereof. During the computation process, we utilise direct linear discriminant analysis (DLDA) to solve the encountered SSS problem. For evaluating the performances of \textit{Discriminative PCA} in face recognition, we analytically compare it with DLAD and PCA on four well known facial databases, they are PIE, FERET, YALE and ORL respectively. Results in accuracy and running time obtained by nearest neighbour classifier are compared when different number of training images per person used. Not only the superiority and outstanding performance of \textit{Discriminative PCA} showed in recognition rate, but also the comparable results of running time.
\end{abstract}


% \linenumbers

\section{Introduction}\label{introduction}

Principal component analysis (PCA) and linear discriminant analysis (LDA) are two of the most popular linear dimensionality reduction approaches. Due to their effectivenesses in feature exraction, PCA, LDA and their variants have been continuious developed and applied into numerous applications in various areas involving pattern recognition, computer vision, industrial engineering and data analysis, etc. For illustrations, a novel variant of PCA, namely the adaptive block sparse PCA based on penalized SVD is proposed to deduce  a new multiple-set canonical correlation analysis (mCCA) method, which is applied to the problem of multi-subject fMRI data sets analysis in \cite{Seghouane18}. And two multilinear extensions of PCA is investigated in \cite{Pacella18} with application to an emission control system. Whereas LDA is often used in field of classifications, such as the application in text classification based on  self-training and LDA topic models is introduced in \cite{Pavlinek17} and the application of facial exprssion classification using LDA and threshold SVM is studied by literature \cite{Shah17}. For accelerating the convergence rate of the incremental LDA algorithm, paper \cite{Ghassabeh15} derives new algorithms by optimizing the step size in each iteration using steepest descent and conjugate direction methods. Besides these applications, literatures \cite{Wu17,Stuhlsatz12} root deep learning techniques into LDA to learn non-linear transformations.

Generally speaking, PCA is trying to find an orthogonal transformation to convert a set of observations of possibly correlated variables into a set of most scattered values of linearly uncorrelated variables. Large number of theoritical analysis and real applications can prove that PCA is simple and efficient, however PCA concerns with the overall sample data, which is a drawback of limiting the same classification rate when changing different viewpoint. Whereas LDA performs excellent separable performances by maximising  between-class distance simultaneously minimising within-class distance however under much expensive computational cost. Owing to the importance of feature extraction in numerous areas, design more effective approach such that exploit advantages of existing methodologies and overcome their shortcomings is promising. In this sense, we aim to enhance separability of PCA by LDA in this paper and at the same time reserve the outstanding performance of low computational complexity. For completing our goal, we need to perform PCA on a discriminative matrix which can be calculated by LDA. However, there are two severe problems exist in regular LDA. One is the computational difficulties and another one appears in the cases that number of observations is less than their dimensions, which is so called \textit{small sample size}-SSS problem. Therefore in order to utilise LDA effectively, we need to find a reasonable solution to resolve the both problems. Thanks to the extensive applications of LDA, there are many scientists focus on how to solve SSS probelm. One of the popular solutions is so called \textit{PCA plus LDA}, which is a method apply PCA as pre-processing step. The relevant theoretical foundation and its applications can be found in \cite{PCA-LDA03, KPCA-LDA04, Sun18}. But incompatibility is a potential problem of \textit{PCA plus LDA}, which may lead to PCA discard dimensions that contain important discriminative information for LDA. Therefore another more feasible and effective approach is designed by \cite{DLDA01, Lu03}, which is named by direct linear discriminant analysis (DLDA). The proposed DLDA algorithms can accept high-dimensional data input and optimize Fisher's criterion directly without any feature extraction operation by discarding the null space of between-class matrix and meantime keeping the null space of within-class matrix to tackle the incompatibility problem of  \textit{PCA plus LDA}. Due to its effectiveness, researchers adopt DLDA in diverse applications and receive outperforming results in \cite{Portillo17, Meshgini13}. Therefore back to our original intention, in this paper we adopt DLDA to solve SSS problem during the implementation process. 

%\hfill mds
 
%\hfill Date
\subsection{Related Works}
The essential problem of \textit{Discriminative PCA} is how to build the computational framework in order the derived subspace to possess the superiorities of LDA and PCA and overcome their limitations simultaneously. According to the fact that classifier fusion is promising research directions in the field of pattern recognition and computer vision. It seems reasonable to expect that a better performance could be obtained by combining the resulting classifiers. Then idea of fusing PCA and LDA hence becomes feasible. Bsaed on this consideration, some relevant combining techiniques are designed for applications in diverse areas involving face recognition, verification, re-identification, classification and fault detection \cite{Oh13, Marcialis02, Sadeghi10, Borade16, Sharma06, Deng17}. Authors fuse PCA-LDA in data preprocessing part for extraction of facial features in \cite{Oh13} by integrating two covariance matrices together into a single covariance matrix. The fused subspce is expected to preserve the nature of both subfaces of PCA and LDA in hence to improve its performances. However during the computation process, the two covariance matrices are calculated directly on original data matrix which is constructed by highly dimensional vectors. This is an intractable task especially in calculation of corresponded eigenvalues. Similar problems occur in the other literatures, despite the fusion strategies perform outstanding results, there still exist two serious drawbacks. One is the expensive computation cost and the other one is the easy occurrence of SSS problem during LDA procedure. Take ORL facial database \cite{ORL} for an illustation, there are 40 people totally and 10 facial images per person with size of pixels $112\times 92$, then each image can be seen as a point in 10304-dimensional space. If we choose 5 training images per individual in face recognition, then the size of training matrix is $10304\times 200$. In this case, the fusion of PCA and LDA will be conducted directly in two $10304\times 10304$ covariances matrices so that the computational cost becomes very expensive which will lead to application difficulties. The SSS probelm that comes with it will cause inefficiency. Therefore fuse PCA and LDA directly cannot meet our purpose, other thoughts to enhance discriminant information for PCA need to be developed. Not like the fusion classifiers, novel methods incorporating the discriminant constraints inside for non-negative matrix factorization (NMF) and kernel NMF (KNMF) are proposed in  \cite{Zafeiriou06, Liang10}. Which inspires us to discover discriminant projections for sample data after projection to the obtained low-dimensional subspace. For guaranteeing outperforming properties with low computational cost and solving the potential SSS problem, we enhance discriminant information inside of PCA procedure by adopting DLDA strategy on a converted small size matrix with respect to original sample matrix. For understanding \emph{Discriminative PCA}, the next section will introduce the processes of feature extraction in face recognition by using of PCA and LDA.

% needed in second column of first page if using \IEEEpubid
%\IEEEpubidadjcol
\section{PCA and LDA}
Through introduction, it is now clear that \textit{Discriminative PCA} is such process that performing PCA on a discriminative matrix which is computed by LDA. Therefore in order to come into our novel approach better, at the beginning we briefly introduce the schemes of PCA and LDA in feature extraction of face recognition. Firstly the symbols and their descriptions are concluded in table \ref{symbols}.

\begin{table*}[ht] 
\centering
\begin{tabular}{|c|c|c|}
\hline
Symbols  & Descriptions  & Dimensions\\
\hline
$\Omega=\{\omega_{ij}\}$    & \tabincell{c}{training matrix consisting of the set of input \\training images, one column represents one facial image}  & $MN\times cl$ \\
\hline
$\omega_{ij}$  & jth face of ith individual & $MN\times 1$ \\
\hline
$MN$  & dimensions  of $\omega_{ij}$  & scalar\\
\hline
$c$   & the number of persons   & scalar\\
\hline
$l$   & \tabincell{c}{number of training images per person, we assume that \\each individual has same number of training images}   & scalar\\
\hline
$cl$   & total number of training images   & scalar\\
\hline
$\bf C$   & covariance matrix  & $MN \times MN$\\
\hline
$\bar{\omega}$   & mean of all training samples  & $MN \times1 $\\ 
\hline
$\bar{\omega_i}$  & mean of the $i$th-class sample  & $MN \times1$\\
\hline
${\bf{S_b}}$  & between-class matrix   & $MN \times MN$\\
\hline
${\bf{S_w}}$  & within-class matrix   & $MN \times MN$\\
\hline
\end{tabular}
\caption{Partial symbols used in this paper.}
\label{symbols}
\end{table*}

The represented symbols listed in table \ref{symbols} are described in details. The input training samples is a set of $c$ observations defined by $${\bf{\Omega}}=\{ \Omega_1, \Omega_2, \dots, \Omega_c\},\ \Omega_i=\{\omega_{i1}, \omega_{i2},\cdots, \omega_{il}\},\ i=1,\cdots, c$$ $c$ and $l$ have their own meanings thereof as table \ref{symbols} describes. Then the training sample matrix can be represented as
\begin{equation}\label{training matrix}
{\bf{\Omega}}=\left[\omega_{11},\dots,\omega_{1l}, \cdots, \omega_{c1},\dots, \omega_{cl}\right]
\end{equation}
which is a real matrix composed by $\omega_{ij}$ sorted as column and the size is $MN\times cl$.

\subsection{PCA implementation by the covariance method}\label{PCA}
PCA \cite{PCA91, PCA10} can be used to reduce dimensions by mapping original sample matrix into a $p$-dimensional feature subspace, where $p \ll MN.$ This is an algorithm based on Karhunen-Lo\`{e}ve transform which is a common orthogonal transform that choose a dimensionality reducing linear projection that maximizes the scatter of all projected samples. Its orthogonal basis functions used in this representation are determined by the covariance function of the process. In detail, the transform was found by expanding the process with respect to the basis spanned by the eigenvectors of the covariance function. The purpose of PCA hence is finding a linear transformation projecting the original sample matrix onto a $p$-dimensional feature subspace $V$ which is defined by the linear transformation as follows,
\begin{equation}
{\bf{Y}}=V^T{\bf{\Omega}}
\end{equation}
where $V\in \RR^{MN\times p}$ is a matrix with orthonormal columns composed by $v_k,\ k=1,\dots,p$. The computation details will be described in next subsequent paragraphs. For making sure projected samples are maximal scatter without correlations, the process implemented on centred sample matrix as the forthcoming paragraph introduces.

Matrix $\Omega$ in table \ref{symbols} is described thereof each column represents an observation which can be seen a point in $MN$-dimensional space. In covariance method, the essential step for PCA is seeking a set of $p$ orthonormal vectors $v_k$, which best represent the distribution of all samples. This process can be satisfied through diagonalizing the covariance matrix of centred sample data by \eqref{covariance matrix}
\begin{equation} \label{covariance matrix}
{\bf{C}}=\frac{1}{cl}\sum_{i=1}^c\sum_{j=1}^l(\omega_{ij}-\bar{\omega})(\omega_{ij}-\bar{\omega})^T
\end{equation}
We learn from table \ref{symbols} that $\bar{\omega}$ is the average of all observations which is defined by $\bar{\omega}=\displaystyle{\frac{1}{cl}}\sum\limits_{i=1}^c\sum\limits_{j=1}^l\omega_{ij}.$ The required feature vectors are the orthonormal eigenvectors of $\bf{C}$ corresponding to $p$ largest eigenvalues, set $A=\displaystyle{\frac{1}{\sqrt{MN}}}[\omega_{11}-\bar{\omega},\cdots, \omega_{1l}-\bar{\omega},\cdots, \omega_{cl}-\bar{\omega}]$ then $v_k$ are chosen so that the diagonal elements of
\begin{equation}
\Lambda= diag \left(\lambda_1,\cdots,\lambda_{MN} \right)=V^T( AA^T)V
\end{equation}
attains larger values. However, the covariance matrix is $MN\times MN$ real symmetric matrix, which is so large therefore easily cause the computation difficulties in calculation of eigenvalues and the corresponding eigenvectors. For overcoming this shortage, more feasible method should be considered, that is decomposing the novel matrix 
\begin{equation*}
\widetilde{\bf{C}}=A^TA
\end{equation*}
by its eigenvectors. the size now is $cl \times cl$, which is much smaller than $\bf{C}$. We assume $u_k$ are the eigenvectors of $\widetilde{\bf{C}}$ corresponding to eigenvalue $\lambda_k$, then the feature vectors $v_k$ derived from can be calculated by $Au_k$ since
\begin{equation}
A^TAu_k=\lambda_ku_k\rightarrow AA^T(Au_k)=\lambda_k(Au_k)
\end{equation}
Now the feature matrix can be calculated by $V=AU$, where $U =[u_1,\dots,u_p]$ which is composed by eigenvectors corresponding to top $p$ largest eigenvalues. PCA has successful application in face recognition which is well known as \emph{eigenfaces}.

\subsection{LDA-fisher method}\label{LDA}
Although the feature space yielded by PCA contains maximal scatter information without correlations among samples, PCA is sensitive to unexpected variations such as illumination, expressions and poses in face recognition because lacking of discriminant information. In these situations, variations between images of the same person are larger than image variations derived from changing person identity. Thus the PCA projections may not be optimal from a discriminant viewpoint. The limitation in separability of \textit{eigenfaces} was overcomed by \cite{LDA97}. Which is so called LDA algorithm, also notified as \emph{fisherface} in face recognition. LDA is a supervised method, then it makes sense to use of labelled information of training samples to build a more reliable method for dimensionality reduction of the feature space. 

Similarly in PCA, feature matrix $W$ of \emph{fisher} LDA can be defined by the following linear transformation
\begin{equation}
{\bf{Y}}=W^T{\bf{\Omega}}
\end{equation}
where $W\in \RR^{MN\times m}$ can be computed in such a way that the ratio of the between-class matrix and the within-class matrix is maximized for all training samples. We use ${\bf{S_b}}$ and ${\bf{S_w}}$ to denote between-class and within-class matrices respectively. The detailed formations are listed in \eqref{discriminant matrices}.
\begin{equation}\label{discriminant matrices}
\begin{aligned}
{\bf{S_b}}&=\sum_{i=1}^c(\bar{\omega_i}-\bar{\omega})(\bar{\omega_i}-\bar{\omega})^T\\
{\bf{S_w}}&=\sum_{i=1}^c\sum_{j=1}^l(\omega_{ij}-\bar{\omega_i})(\omega_{ij}-\bar{\omega_i})^T
\end{aligned}
\end{equation}
where $\bar{\omega_i}=\displaystyle{\frac{1}{l}\sum_{j=1}^l}\omega_{ij},$ then the feature matrix $W$ of \emph{fisher} LDA can be obtained by
\begin{equation}\label{fisher}
\begin{aligned}
W&=\mathop{\arg\max}_{W}  \frac{W^T{\bf{S_b}}W}{W^T{\bf{S_w}}W}\\
&=[w_1, w_2, \cdots, w_m]
\end{aligned}
\end{equation}
where $w_i,\ i=1,\cdots,m$ is the eigenvectors of $\bf{S_b}$ and $\bf{S_w}$ corresponding to the top $m$ largest eigenvalues, i.e. 
\begin{equation*}
{\bf{S_w^{-1}}}{\bf{S_b}}w_i=\lambda_iw_i,\ i=1,2,\dots,m
\end{equation*}
Obviously, feature vectors $w_i$ can be obtained by eigenvalue decomposition of ${\bf{S_w^{-1}}}{\bf{S_b}}$, the size of which is $MN\times MN$. Similarly as in PCA, this is an intractable task. In addition to computational difficulty, one potential situation easily  confronted is that ${\bf{S_w}}$ is always singular especially in face recognition field. This stems from the fact that the number of observations much smaller than their dimensions, i.e. the pixels number, this is so called SSS problem. 

\subsection{Conclusions of PCA and LDA}
To illustrate the properties of PCA and LDA, we conclude their superiorities and shortcomings in this subsection to bring the original intention of \emph{Discriminative PCA}. The key advantage of PCA in face recognition is the low noise sensitivity, low computational cost and high recognition accuracy on ideal facial databases. Compared with LDA, PCA works better in case where number of class is less. However, PCA only consider the most scattered information among all samples, whereas LDA works better with large dataset having multiple classes which stems from the fact that class separability is an important factor while reducing dimensionality. Conversely, LDA has much more expensive computational cost than PCA and usually face to SSS problem. Based on these considerations, construct a novel technique that can preserve both superiorities of PCA and LDA and at the same time overcome their weaknesses is promising and necessary. \emph{Discriminative PCA} derives from this original intention. In the next section, the details of \emph{Discriminative PCA} will be explained well.

\section{Discriminative PCA}
The core idea of \textit{Discriminative PCA} is constructing an algorithm to find a subspace contains discriminative principal components, i.e. \emph{eigen}-subspace that possesses discriminant information. In order to fulfil this purpose, the essential implementation is performing PCA on discriminative matrices. Therefore the main problem need to be solved by \emph{Discriminative PCA} is how to compute discriminative matrices. In one word, similarly as in PCA and LDA, we are aiming to find a feature matrix $\bf{\Xi}$ such that build the following linear transformation 
\begin{equation}\label{projection_dpca}
{\bf{Y}}={\bf{\Xi}}^T\Omega
\end{equation}
where ${\bf{\Xi}}\in \RR^{MN\times p}$ composed by a set of feature vectors $\tilde{v}_k,\ k=1,\dots,p$ which are top $p$ principal components containing discriminant information. The idea of \emph{Discriminative PCA} arises a natural and feasible thought that is implementing PCA to the matrix that contains discriminative information instead of to original sample data. This is a process that use of LDA to enhance separability for PCA. However if we directly adopt LDA, then the conundrums of computational cost and SSS problem will stop us forward. From subsection \ref{LDA}, it is clearly the size of ${\bf{S_b}}$ and ${\bf{S_w}}$ is $MN \times MN$, which is too large to calculating the relevant discriminative matrices. Therefore the primary issue for us is designing a novel small size matrix such that the complexity shall be much lower when perform LDA on it than on original sample matrix. The original training sample matrix is hence converted into
\begin{equation}\label{New_sample}
\widetilde{\mathcal{S}} \coloneqq {\bf{\Omega}}^T{\bf{\Omega}},
\end{equation} 
which is a $cl\times cl$ matrix. $cl$ stands for class number as described in table \ref{symbols}, the size of $\widetilde{\mathcal{S}}$ therefore is much smaller now so that the computation crisis is solved successfully. In order to obtain the discrimination information of $\Omega$ rather than $\widetilde{\mathcal{S}}$, we should find the relationship between them. We start with the calculations  of ${\bf{\widetilde{S}_b}}$ and ${\bf{\widetilde{S}_w}}$ for $\widetilde{\mathcal{S}}$ by    
\begin{equation*}
\begin{aligned}
{\bf{\widetilde{S}_b}}&=\sum_{i=1}^{c} ({\bf{\Omega}}^T\bar{\omega_i}-{\bf{\Omega}}^T\bar{\omega})({\bf{\Omega}}^T\bar{\omega_i}-{\bf{\Omega}}^T\bar{\omega})^T\\  
&={\bf{\Omega}}^T \left[\sum_{i=1}^{c} (\bar{\omega_i}-\bar{\omega})(\bar{\omega_i}-\bar{\omega})^T \right]{\bf{\Omega}}\\ 
\end{aligned}
\end{equation*}
similarly, we have 
\begin{equation*}
\begin{aligned}
{\bf{\widetilde{S}_w}}&=\sum_{i=1} ^{c}\sum_{j=1}^{l}({\bf{\Omega}}^T\omega_{ij}-{\bf{\Omega}}^T\bar{\omega_i})({\bf{\Omega}}^T\omega_{ij}-{\bf{\Omega}}^T\bar{\omega_i})^T\\ 
&={\bf{\Omega}}^T \left[ \sum_{i=1} ^{c}\sum_{j=1}^{l}(\omega_{ij}-\bar{\omega_i})(\omega_{ij}-\bar{\omega_i})^T\right] {\bf{\Omega}}.
\end{aligned}
\end{equation*}
where ${{\bf{\widetilde{S}_b}}}$ and ${{\bf{\widetilde{S}_w}}}$ denote between-class matrix and within-class matrix of $\widetilde{\mathcal{S}}$ respectively. $\bf{S_b}$ and $\bf{S_w}$ are the corresponding discriminative matrices of $\Omega$. Then the relationship between them can be described as
\begin{equation}
\begin{aligned}
{\bf{\widetilde{S}_b}}&={\bf{\Omega}}^T {\bf{S_b}} {\bf{\Omega}} \\
{\bf{\widetilde{S}_w}}&={\bf{\Omega}}^T {\bf{S_w}} {\bf{\Omega}} 
\end{aligned}
\end{equation}
The feature space $\bf{\Xi}$ we committed to find should possess discriminant information, which means that $\bf{\Xi}$ can be obtained based on the optimal subspace $W$ described in formulation \eqref{fisher}. For the reason of computational difficulty, we calculate $\widetilde{W}$ of $\widetilde{\mathcal{S}}$ at the beginning in order to deduce $W$. We firstly give a lemma to explain the relationship between $W$ and $\widetilde{W}$.

\begin{Lemma}\label{novel_w}
Suppose $X,\ A$ are invertible matrices, if $\widetilde{w}$ is an eigenvector of $\widetilde{A}^{-1}\widetilde{B}$, and $\widetilde{A}=X^TAX,\ \widetilde{B}=X^TBX$, then $X\widetilde{w}$ is an  eigenvector of $A^{-1}B$.
\end{Lemma}
\begin{proof}
\begin{equation*}
\begin{aligned}
&\widetilde{B}\widetilde{w}=\lambda \widetilde{A} \widetilde{w}\\
&\Rightarrow (X^TBX)\widetilde{w}=\lambda X^TAX \widetilde{w}\\
&\Rightarrow (X^TA)^{-1}(X^TB)X\widetilde{w}=\lambda X \widetilde{w}\\
&\Rightarrow A^{-1}(X^T)^{-1}(X^T)B(X\widetilde{w})=\lambda (X \widetilde{w})\\
&\Rightarrow A^{-1}B(X\widetilde{w})=\lambda (X \widetilde{w})
\end{aligned}
\end{equation*}
\end{proof}
According to Lemma \ref{novel_w}, $W$ can be obtained by
\begin{equation}\label{discriminant_novel}
W={\bf{\Omega}}\widetilde{W}
\end{equation}
where $\widetilde{W}\in \RR^{cl\times m}$ is constructed by eigenvectors $\widetilde{w}_k,\ k=1,\dots,m$ corresponding to top $m$ largest eigenvalues derived from $\displaystyle{ {\bf{\widetilde{S}_w}}^{-1}{\bf{\widetilde{S}_b}}\widetilde{W}}=\lambda \widetilde{W}$. But non-singularity is a necessary requirement for all the relevant computations above mentioned. However, particularly in face recognition it is very difficult to guarantee sample matrix and within-class matrix are non-singular, i.e. SSS problem. Therefore, we need to find an appropriate approach to solve this task. In addition to SSS problem, we find that the elemental values of ${\bf{\widetilde{S}_w}}$ and ${\bf{\widetilde{S}_b}}$ are too large to getting correct results. Before settle SSS problem, we first design regularization strategies to preserve that elemental values in an appropriate range. There are two ways to carry out this goal, which are described as below formula:
\begin{equation}\label{meanvalue_rule}
\left \{
\begin{array}{l}
{\bf{\widetilde{S}_w}}={\bf{\widetilde{S}_w}}./\overline{{\bf{\widetilde{S}}}}_w\\
{\bf{\widetilde{S}_b}}={\bf{\widetilde{S}_b}}./\overline{{\bf{\widetilde{S}}}}_b
\end{array}
\right.
\end{equation}
and 
\begin{equation}\label{maxvalue_rule}
\left \{
\begin{array}{l}
{\bf{\widetilde{S}_w}}={\bf{\widetilde{S}_w}}./\max ({\bf{\widetilde{S}_w}})\\
{\bf{\widetilde{S}_b}}={\bf{\widetilde{S}_b}}./\max ({\bf{\widetilde{S}_b}})
\end{array}
\right.
\end{equation}
where symbol $./$ stands for each element of matrix in numerator divide its denominator. $\overline{\bullet}$ is the mean value of all elements of matrix $\bullet$ and $\max (\bullet)$ is the maximal element of $ \bullet$. We call regularization showed in \ref{meanvalue_rule} as mean value rule and other one of \ref{maxvalue_rule} is maximum rule.

Now it is time to face SSS problem occurred during the process of calculate $\widetilde{W}$. SSS problem still exit because $$\mathop{rank}({\bf{\widetilde{S}_w}})\leq \min \{\mathop{rank}({\bf{\Omega}}),\ \mathop{rank}({\bf{S_w}}) \}$$ 
As mentioned in introduction, DLDA is used to solve SSS problem in this paper to resolve the defect of losing important discriminant information deduced by \emph{PCA plus LDA}. The most important innovation of DLDA is discard the null space of $\bf{\widetilde{S}_b}$ rather than discarding the null space of $\bf{\widetilde{S}_w}$. The benefit of this way is can reserve the most discriminative information from the subspace ${\bf{B}}' \cap {\bf{A}}$, where ${\bf{B}}'$ is the complementary space of \textbf{B}, which is null space spanned by eigenvectors corresponding to zero eigenvalues of $\bf{\widetilde{S}_b}$. \textbf{A} is the space which is spanned by eigenvectors corresponding to the relevant smaller eigenvalues of $\bf{\widetilde{S}_w}$. Before specifically explain the implementation process of DLDA, relevant denotes are given by here: $\bf{E_b}$ is a space spanned by eigenvectors of $\bf{\widetilde{S}_b}$ corresponding to its all eigenvalues which are used to construct diagonal matrix ${\bf{\Lambda}}_b$. Firstly we discard the eigenvectors of zero eigenvalues from $\bf{E_b}$, then $\hat{\bf{E}}_b$ stands for the space spanned by the remaining eigenvectors and $\hat{\bf{\Lambda}}_b$ is a diagonal matrix of which diagonal elements are the corresponding remaining eigenvalues. \vspace{0.1cm} Thereby ${\bf{B}}'$ can be calculated through the following formulation, in this way we have ${\bf{B}}'^T{\bf{\widetilde{S}_b}}{\bf{B}}'=\bf{I}_n$, where $\bf{I}_n$ is an identity matrix with size $n\times n$ and $n$ is the number of the remaining eigenvectors.
\begin{equation}\label{B}
{\bf{B}}'=\hat{\bf{E}}_b\hat{\bf{\Lambda}}_b^{-1/2}
\end{equation}
Based on formula \eqref{B}, the intersect subspace ${\bf{B}}' \cap {\bf{A}}$ can be achieved by the next steps. Firstly diagonalise ${\bf{B}}'^T{\bf{\widetilde{S}_w}}{\bf{B}}'$ to get the eigenvectors' space ${\bf{E}}_w$ and the corresponding diagonal matrix ${\bf{\Lambda}}_w$ which is composed by corresponding eigenvalues. The second step is discard the eigenvectors corresponding to largest eigenvalues. We use  $\hat{\bf{E}}_w$ and $\hat{\bf{\Lambda}}_w$ stand for the spaces spanned by the remaining eigenvectors and eigenvalues after discarding. Now $\widetilde{W}$ can be calculated by 
\begin{equation}
\widetilde{W}={\bf{B}'}\hat{\bf{E}}_w\hat{\bf{\Lambda}}_w^{-1/2}
\end{equation}
Eventually according to \eqref{discriminant_novel}, we get the discriminative matrix of $\bf{\Omega}$ by
\begin{equation}\label{W}
W={\bf{\Omega}}\widetilde{W}
\end{equation}

\renewcommand{\algorithmicrequire}{ \textbf{Input:}} %Use Input in the format of Algorithm
\renewcommand{\algorithmicensure}{ \textbf{Output:}} %UseOutput in the format of Algorithm

\begin{algorithm*}[htbp] 
\caption{Pseudocode for calculating feature subspace $\bf{\Xi}$ by \emph{Discriminative PCA} algorithm.} 
\label{alg_DPCA} 
\begin{algorithmic} 
\REQUIRE ~~\\ 
Training facial images set ${\bf{\Omega}}$\\
\ENSURE ~~\\ 
Feature subspace $\Xi$ that used for linear transformation
\end{algorithmic}
\renewcommand{\algorithmicrequire}{ \textbf{Process:}} 
\begin{algorithmic} 
\REQUIRE ~~\\ 
\STATE Step 1. Calculate $\bf{\widetilde{S}_w}$ and $\bf{\widetilde{S}_b}$ of ${\bf{\Omega}}^T{\bf{\Omega}}$; 
\vspace{0.1cm}
\STATE Step 2. Regularize $\bf{\widetilde{S}_w}$ and $\bf{\widetilde{S}_b}$;
\vspace{0.1cm}
\STATE Step 3. Calculate the eigenvectors of $\bf{\widetilde{S}_b}$ corresponding to non-zero eigenvalues: ${\bf{\hat{E}}}_b=[{e_b}_1,\dots,{e_b}_n],\ {\bf{\hat{\Lambda}}}_b^{-1/2}=[{\lambda_b}_1^{-1/2},\dots,{\lambda_b}_n^{-1/2}]$;
\vspace{0.1cm}
\STATE Step 4. Let ${\bf{B'=\hat{E}}}_b{\bf{\hat{\Lambda}}}_b^{-1/2}$, then calculate the eigenvectors ${\bf{E}}_w$ of ${\bf{B'}}^T{\bf{\widetilde{S}}}_w{\bf{B'}}$;
\vspace{0.1cm}
\STATE Step 5. Discard the eigenvectors of ${\bf{E}}_w$ with respect to the largest eigenvalues to obtain ${\bf{\hat{E}}}_w=[{e_w}_1,\dots,{e_w}_m]$;
\label{step_5}
\vspace{0.1cm}
\STATE Step 6. Calculate feature subspace $\widetilde{W}$ of ${\bf{\Omega}}^T{\bf{\Omega}}$ by $\widetilde{W}={\bf{B}}'{\bf{\hat{E}}}_w{\bf{\hat{\Lambda}}}_w^{-1/2},\ {\bf{\hat{\Lambda}}}_w^{-1/2}=[{\lambda_w}_1^{-1/2},\dots, {\lambda_w}_m^{-1/2}]$;
\vspace{0.1cm}
\STATE Step 7. Calculate discriminative matrix $W$ of ${\bf{\Omega}}$ by $W={\bf{\Omega}}\widetilde{W}$;
\vspace{0.1cm}
\STATE Step 8. Compute the covariance of centred discriminative matrix $W$ through ${\bf{C}_W}=\displaystyle{\frac{1}{m}}(W-\bar{W})^T(W-\bar{W})$;
\vspace{0.1cm}
\STATE Step 9. Select the eigenvectors $[{e_{\bf{C}}}_1,\dots,{e_{\bf{C}}}_p]$ with the top $p$ largest eigenvalues of $\bf{C}$;
\label{step_9}
\vspace{0.1cm}
\STATE Step 10. Normalise $[{e_{\bf{C}}}_1,\dots,{e_{\bf{C}}}_p]$ to obtain feature subspace ${\bf{\Xi}}$ which is composed by $\left[\displaystyle{\frac{{e_{\bf{C}}}_1}{\Vert {e_{\bf{C}}}_1 \Vert}},\dots,\displaystyle{\frac{{e_{\bf{C}}}_p}{\Vert {e_{\bf{C}}}_p \Vert}} \right]$.
\vspace{0.1cm}
\end{algorithmic}
\end{algorithm*}

The subsequent task of \textit{Discriminative PCA} is perform PCA on discriminative matrix $W$ obtained by \eqref{W} to get the feature space $\bf{\Xi}$. As steps explained in subsection \ref{PCA}, we firstly construct the covariance matrix ${\bf{C}}_W$ based on centred $W$ and then select the orthonormal eigenvectors with top $p$ largest eigenvalues of ${\bf{C}}_W$ to construct the feature space $\bf{\Xi}$. Then we can extract the discriminative principal features by mapping all training samples to $\bf{\Xi}$ through \eqref{projection_dpca}. The detailed process is described in the pseudocode of \emph{Discriminative PCA} algorithm \ref{alg_DPCA}.

\section{Experimental Results}
For evaluating the effectiveness of \emph{Discriminative PCA}, in this section we compare the performances in recognition accuracy and running time of \emph{Discriminative PCA} with DLDA and PCA in face recognition on four facial databases, they are CMU PIE \cite{PIE}, FERET\cite{FERET}, YALE \cite{YALE} and ORL respectively. We use mean value rule \eqref{meanvalue_rule} to complete the regularization step of \emph{Discriminative PCA} on the four experimental facial databases. All the face images we used in experiments are gray-scale and dealt with aligning by the locations of eyes. And only facial part image reserved after cropping. These four facial databases have their various properties. For an illustration, CMU PIE contains different factors in pose, illumination and expression. Specifically over 40,000 facial images of 68 individuals are collected. Each person is imaged across 13 different poses, under 43 different illumination conditions, and with 4 different expressions. Figure \ref{PIE} shows the partial images of CMU PIE with brief description. 
\begin{figure}[!htbp]
\centering
\includegraphics[ width=7cm]{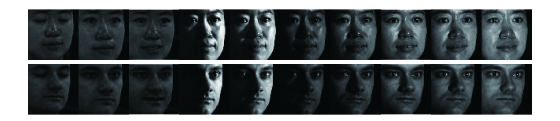}
\caption{We randomly select 15 people and 49 images for each person with the size of $64\times64$ pixels under much different situations involving illumination and expression from pose-No. 05 of PIE database.}
\label{PIE}
\end{figure}

FERET database derives from the Face Recognition Technology (FERET) program, which is a large database of facial images, divided into development and sequestered portions. There are color- and gray-version FERET databases containing more than 10,000 images under various situations involving pose, age, expressions, etc. We choose 50 people and 7 images per each individual of gray-version FERET database in experiments. These images are tiff format with pixels $80\times 80$ and partial of them are displayed in Figure \ref{FERET}.
\begin{figure}[!htbp]
\centering
\includegraphics[width=7cm]{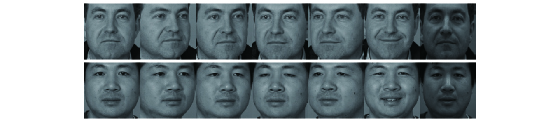}
\caption{Partial images of Gray FERET face database.}
\label{FERET}
\end{figure}

YALE face database contains 165 images of 15 individuals. There are 11 images per each person, one per different facial expression or configuration as partial images showed in figure \ref{YALE}: center-light, w/glasses, happy, left-light, w/no glasses, normal, right-light, sad, sleepy, surprised, and wink. We select all facial images of YALE database to complete our experiments in this part.
\begin{figure}[!htbp]
\centering
\includegraphics[width=7cm]{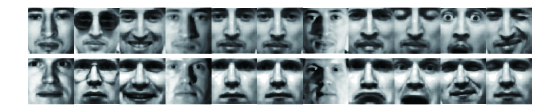}
\caption{Partial images of YALE face database, in which all individuals are chosen for experiments and the size of each image is $80\times 80$ pixels.}
\label{YALE}
\end{figure}

Figure \ref{ORL} shows the partial images of ORL facial database collected under relatively ideal situations, which means that images of ORL contain less changes compared with PIE, FERET and YALE.
\begin{figure}[!htbp]
\centering
\includegraphics[width=7cm]{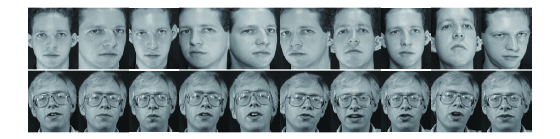}
\caption{Partial images of ORL, this database composed by 40 individuals and 10 images per person with the size of $112\times 92$ pixels.}
\label{ORL}
\end{figure}

Nearest-neighbour classifier is used to finish recognition step in our experiments and running time is the average value by running 20 times. More reliable results can be obtained through using different number of training images per person. In this way we can analytically compare the performances of \emph{Discriminative PCA}, DLDA and PCA more  comprehensive. The subscript of denotion `Property$\bf{_{Training \ number}}$' showed in tables \ref{tab_PIE}-\ref{tab_ORL} stands for this index. There are three cases in PIE database we design for experiments. It is clearly observe from table \ref{tab_PIE} that whenever training number is small or large (from 5 to 20 images for each person), \emph{Discriminative PCA} always be far superior to PCA and DLDA in recognition accuracy. Particularly when the training number is small, the advantage is more significant. And in running time, \emph{Discriminative PCA} has a comparable results with PCA, in special case, it is even faster than PCA. Among the three algorithms, results of DLDA and PCA are consistent with our analysis that DLDA generally performs better in recognition rate whereas PCA is dominant in computational cost.
\begin{table*}[htbp] 
\centering
\begin{tabular}{cccc}
\hline 
Property$\bf{_{Training \ number}}$ \quad \quad  & PCA \quad \quad & DLDA \quad \quad & \textit{Discriminative PCA} \\
\hline 
Accuracy$\bf{_5}$\quad \quad & 31.36\%  \quad  \quad & 26.82\%  \quad \quad & \textcolor{red}{43.79\% }  \\
Accuracy$\bf{_{15}}$\quad \quad & 53.53\%  \quad \quad & 93.53\%   \quad \quad & \textcolor{red}{95.88\% }  \\
Accuracy$\bf{_{20}}$\quad \quad & 69.66\%  \quad  \quad & 93.33\%   \quad  \quad & \textcolor{red}{95.86\% }  \\
Running time$\bf{_5}$\quad \quad & 0.668472 \quad \quad & 0.913645 \quad \quad & \textcolor{red}{0.667416} \\
Running time$\bf{_{15}}$\quad \quad & \textcolor{red}{0.706966} \quad \quad & 0.961559  \quad  \quad & 0.723554 \\
Running time$\bf{_{20}}$\quad \quad & \textcolor{red}{0.721420} \quad \quad & 1.005333 \quad \quad & 0.757531 \\
\hline 
\end{tabular}
\caption{Results of recognition accuracy and running time of PCA, DLDA and \emph{Discriminative PCA} when different training images per person used on PIE.}
\label{tab_PIE}
\end{table*}

Since images for each person we choose in PIE under much different situation in illumination of the same pose, results in table \ref{tab_PIE} indicate that \emph{Discriminative PCA} resolve the problem of sensitive to illumination for PCA by discriminant enhancement and at the same time reserve the low computational complexity. For obtaining intuitively view, we display first ten basis images of their own feature subspaces in figure \ref{basis_PIE}. We find that in PCA, the basis images contain much more illumination information than \emph{Discriminative PCA} and DLDA.
\begin{figure}[!htbp]
\centering
\includegraphics[ width=7cm]{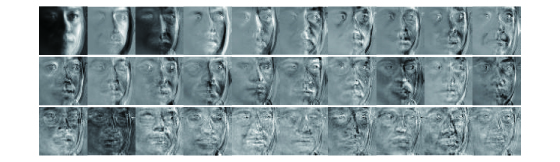}
\caption{From top row to bottom, the basis images on PIE database obtained sequentially by PCA, DLDA and \textit{Discriminative PCA}.}
\label{basis_PIE}
\end{figure}

Different from in PIE, images of FERET we select are mainly focus on variations in  poses and expressions. Under these factors, results of table \ref{tab_FERET} also verify the similar conclusion that \emph{Discriminative PCA} is far ahead on recognition rate than DLDA and PCA. More worth mentioning is the stable performance of \emph{Discriminative PCA} even when training number is very small (2 images for each one). With increasing the training images, recognition rates of DLDA and \emph{Discriminative PCA} higher, which is a different phenomenon compared with PCA. Even the training number is large, PCA can not perform well when face images much variant in illumination, pose and expressions, etc. It proves that PCA has good properties under ideal situations such in table \ref{tab_ORL} shows. 
\begin{table*}[htbp]
\centering
\begin{tabular}{cccc}
\hline 
Property$\bf{_{Training \ number}}$ \quad \quad & PCA \quad \quad & DLDA \quad \quad & \textit{Discriminative PCA} \\
\hline 
Accuracy$\bf{_2}$\quad \quad & 55.20\%  \quad \quad & 28.00\%   \quad \quad & \textcolor{red}{73.20\% }  \\
Accuracy$\bf{_3}$\quad \quad & 45.00\%  \quad \quad & 63.00\%   \quad \quad & \textcolor{red}{74.50\% }  \\
Accuracy$\bf{_4}$\quad \quad & 52.00\%  \quad \quad & 76.00\%   \quad \quad & \textcolor{red}{79.33\% }  \\
Accuracy$\bf{_5}$\quad \quad & 40.00\%  \quad \quad & 80.00\%   \quad \quad & \textcolor{red}{90.00\% }  \\
Running time$\bf{_2}$\quad \quad & \textcolor{red}{0.223107} \quad \quad & 0.798677 \quad \quad & 0.366499 \\
Running time$\bf{_3}$\quad \quad & \textcolor{red}{0.239527} \quad \quad & 0.848357 \quad \quad & 0.243254 \\
Running time$\bf{_4}$\quad \quad & \textcolor{red}{0.241580} \quad \quad & 0.824079  \quad \quad & 0.275457 \\
Running time$\bf{_5}$\quad \quad & \textcolor{red}{0.258035} \quad \quad & 0.828137  \quad \quad & 0.273732 \\
\hline 
\end{tabular}
\caption{Results of recognition accuracy and running time of PCA, DLDA and \emph{Discriminative PCA} when different training images per person used on FERET.}
\label{tab_FERET}
\end{table*}

Similar results can be concluded through observing figure \ref{basis_FERET}, which shows the first ten basis images of three feature subspaces. The basis images of PCA contain more variant information in pose and expression than \emph{Discriminative PCA} and DLDA.
\begin{figure}[!htbp]
\centering
\includegraphics[ width=7cm]{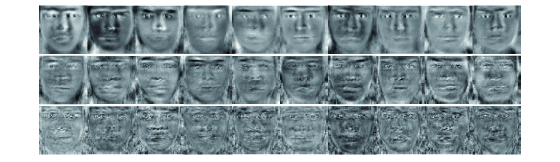}
\caption{Basis images on FERET database calculated by various algorithms: first row is PCA, second row is DLDA and last row is \textit{Discriminative PCA}.}
\label{basis_FERET}
\end{figure}

Different from other facial databases we used, images of YALE contain occlusion changes such as wearing glasses or not, furthermore the expression changes more significant. Without exception
 \emph{Discriminative PCA} still outperform in recognition accuracy and also in running time in this case. When the number of images used for training is larger, PCA performs better than DLDA, details displayed in table \ref{tab_YALE}. 

\begin{table*}[htbp]
\centering
\begin{tabular}{cccc}
\hline 
Property$\bf{_{Training \ number}}$ \quad  & PCA \quad & DLDA \quad & \textit{Discriminative PCA} \\
\hline 
Accuracy$\bf{_3}$\quad \quad & 62.86\%  \quad \quad & 66.67\%   \quad \quad & \textcolor{red}{75.24\% }  \\
Accuracy$\bf{_4}$\quad \quad & 57.14\%  \quad \quad & 72.38\%   \quad \quad & \textcolor{red}{78.10\% }  \\
Accuracy$\bf{_5}$\quad \quad & 81.11\%  \quad \quad & 77.78\%   \quad \quad &\textcolor{red}{82.22\% }  \\
Accuracy$\bf{_7}$\quad \quad & 91.11\%  \quad \quad & 80.00\%   \quad \quad &\textcolor{red}{95.56\% }  \\
Running time$\bf{_3}$\quad \quad & 0.221750 \quad \quad & 0.769583 \quad \quad & \textcolor{red}{0.216452} \\
Running time$\bf{_4}$\quad \quad & 0.246141 \quad \quad & 0.847130 \quad \quad & \textcolor{red}{0.245404} \\
Running time$\bf{_5}$\quad \quad & \textcolor{red}{0.238664} \quad \quad & 0.781730  \quad \quad & 0.245448 \\
Running time$\bf{_7}$\quad \quad & 0.228120 \quad \quad & 0.764449  \quad \quad & \textcolor{red}{0.224151} \\
\hline 
\end{tabular}
\caption{Results of recognition accuracy and running time of PCA, DLDA and \emph{Discriminative PCA} when different training images per person used on YALE.}
\label{tab_YALE}
\end{table*}

Among the four databases, only on ideal facial database ORL, PCA has stable performance that with training number increase, the recognition rate higher thereof. And in this case when training number is small \emph{Discriminative PCA} still far superior than PCA, whereas if training number is large the three approaches have same good performances as table \ref{tab_ORL} shows.
\begin{table*}[htbp]
\centering
\begin{tabular}{cccc}
\hline 
Property$\bf{_{Training \ number}}$ \quad \quad & PCA \quad \quad & DLDA \quad \quad & \textit{Discriminative PCA} \\
\hline 
Accuracy$\bf{_3}$\quad \quad & 85.36\%  \quad \quad & 77.86\%   \quad \quad & \textcolor{red}{90.36\% }  \\
Accuracy$\bf{_5}$\quad \quad & 93.50\%  \quad \quad & 89.00\%   \quad \quad & \textcolor{red}{95.00\% }  \\
Accuracy$\bf{_7}$\quad \quad & 95.83\%   \quad \quad & 95.83\%   \quad \quad & 95.83\%   \\
Running time$\bf{_3}$\quad \quad & 0.795025 \quad \quad & 2.389738 \quad \quad & \textcolor{red}{0.777746} \\
Running time$\bf{_5}$\quad \quad & \textcolor{red}{0.878740} \quad \quad & 2.351540 \quad \quad & 0.947311 \\
Running time$\bf{_7}$\quad \quad & \textcolor{red}{0.851204} \quad \quad & 2.437706  \quad \quad & 0.906449 \\
\hline 
\end{tabular}
\caption{Results of recognition accuracy and running time of PCA, DLDA and \emph{Discriminative PCA} when different training images per person used on ORL.}
\label{tab_ORL}
\end{table*}

\section{Conclusions and Future Work}
We propose a novel feature extraction approach denoted by \emph{Discriminative PCA} in this paper. The main purpose of \emph{Discriminative PCA} is trying to find a feature subspace contains discriminant principal components. For achieving our goal, LDA is used to enhance the separability for PCA. The core idea of \emph{Discriminative PCA} is performing PCA on discriminative matrix. During the implementation process, we adopt DLDA strategy to solve SSS problem when compute ${\bf{\widetilde{S}}_b}$ and ${\bf{\widetilde{S}}_w}$ of converted training sample matrix $\bf{\Omega^T\Omega}$, which is a trick that simultaneously reduce the computational complexity and solve SSS problem for the phase of calculating discriminative matrix. The superiorities of \emph{Discriminative PCA} have been proved through experimental results on four popular facial databases in recognition rate and average running time. We remark that the number of discarded eigenvectors with top $m$ largest eigenvalues in step \ref{step_5} and basis image number $p$ in step \ref{step_9} of algorithm \ref{alg_DPCA} is important to face recognition rate. Therefore how to select appropriate $m,\ p$ is a task need to be completed. Another point is that in the experimental part, we regularize discriminant matrices by mean value rule. If we use of maximum rule, the results in running time is faster than use of mean value rule whereas the recognize accuracy relatively lower a little bit than using mean value rule, which means that the performance of face recognition for \emph{Discriminative PCA} is not so outstanding among the compared approaches PCA and DLDA.

On the other hand, the proposed \emph{Discriminative PCA} is a linear pattern recognition approach, however many pattern samples lie on non-linear manifold to lead that linear models can not extract and represent non-linear information thereof well, adding discriminant information to non-linear approaches, such as kernel principal component analysis in order to improve the performances therefore is our next work.

\section*{Acknowledgement}

This work is jointly supported by the grants from Guangxi  Science and Technology Base and Talent Specialized Project No. 2018AD19038 and Doctoral Scientific Research Foundation No. GLUTQD2017142 of Guilin University of Technology. 
\section*{References}

\bibliography{DiscriminativePCA-PR}

\end{document}